\documentclass{article}

\usepackage[final]{neurips_2019}

\usepackage[utf8]{inputenc} % allow utf-8 input
\usepackage[T1]{fontenc}    % use 8-bit T1 fonts
\usepackage{hyperref}       % hyperlinks
\usepackage{url}            % simple URL typesetting
\usepackage{booktabs}       % professional-quality tables
\usepackage{amsfonts}       % blackboard math symbols
\usepackage{nicefrac}       % compact symbols for 1/2, etc.
\usepackage{microtype}      % microtypography

\usepackage{amsmath,amssymb,amsthm}
\usepackage{mathtools}
\usepackage{bm}
\usepackage[usenames,dvipsnames]{xcolor}
\usepackage{tikz}
\usepackage{pgfplots}
\usepgfplotslibrary{fillbetween}

\newcommand{\defeq}{\vcentcolon=}
\def\*#1{\bm{#1}}

\newtheorem{definition}{Definition}
\newtheorem{proposition}{Proposition}

\definecolor{mydarkblue}{rgb}{0,0.08,0.45}
\hypersetup{ %
	pdftitle={},
	pdfauthor={},
	pdfsubject={Proceedings of the International Conference on Machine Learning 2019},
	pdfkeywords={},
	pdfborder=0 0 0,
	pdfpagemode=UseNone,
	colorlinks=true,
	linkcolor=mydarkblue,
	citecolor=mydarkblue,
	filecolor=mydarkblue,
	urlcolor=mydarkblue,
	pdfview=FitH}
\newcommand{\figref}[1]{\hyperref[#1]{Figure \ref*{#1}}}

\title{Strong Log-Concavity Does Not Imply Log-Submodularity}

\author{%
  Alkis Gotovos \\
  MIT CSAIL \\
  \texttt{alkisg@mit.edu}
}

\begin{document}

\maketitle

\begin{abstract}
  We disprove a recent conjecture regarding discrete distributions and their generating polynomials stating that strong log-concavity implies log-submodularity.
\end{abstract}

\section{Introduction}
Given a ground set $V = \{1,\ldots, n\}$, we consider distributions $p : 2^n \to \mathbb{R}$ over subsets of $V$, or equivalently, distributions over $n$ binary random variables, which can be represented by the corresponding multi-affine generating polynomial $g_p$ in $n$ variables $x_1,\ldots,x_n \in \mathbb{R}$ as follows,
\begin{align*}
g_p(\*x) = \sum_{S \subseteq [n]} p(S)\*x^S,
\end{align*}
where $\*x = (x_1,\ldots,x_n)$ and $\*x^S \defeq \prod_{v \in S}x_i$.
In what follows, we will use $\partial_{x_i} g$ to denote the partial derivative of $g$ with respect to $x_i$.

\vspace{1em}
\begin{definition}[\citealp{gurvits10,anari19,branden19}]
A polynomial $g$ is called strongly log-concave (SLC) if, for any $k \geq 0$ and any integer sequence $1 \leq i_1,\ldots,i_k \leq n$, the derivative $\partial x_{i_1}\cdots \partial x_{i_k} g$ is log-concave on $\mathbb{R}_{>0}^n$.
\end{definition}

\vspace{1em}
\begin{definition}
A distribution $p$ is called log-submodular, or equivalently, is said to satisfy the negative lattice condition (NLC), if for any $S, T \subseteq V$,
\begin{align*}
p(S)p(T) \geq p(S\cup T)p(S \cap T).
\end{align*}
We will call a generating polynomial $g_p$ log-submodular if the corresponding distribution $p$ is log-submodular.
\end{definition}

Recently, \cite{robinson19} conjectured that strong log-concavity implies log-submodularity for any generating polynomial.
We present here a counterexample that disproves this conjecture.

\section{Counterexample}
Consider the discrete distribution $p$ over ground set $V = \{1, 2, 3\}$, represented by the following generating polynomial,
\begin{align}
g_p(x, y, z) = \frac{1}{22}(4 + 3(x+y+z) + 3(xy + xz + yz)), \label{eq:poly}
\end{align}
for all $x, y, z \in \mathbb{R}$.

\begin{proposition}
The polynomial $g_p$ in \eqref{eq:poly} is not log-submodular.
\end{proposition}

\begin{proof}
For $S = \{1\}$ and $T = \{2\}$, we have
\begin{align*}
p(S)p(T) = p(\{1\})p(\{2\}) = \frac{3}{22}\cdot \frac{3}{22} < \frac{3}{22}\cdot \frac{4}{22} = p(\{1, 2\})p(\varnothing) = p(S\cup T)p(S\cap T).
\end{align*}
\end{proof}

\begin{proposition}
The polynomial $g_p$ in \eqref{eq:poly} is strongly log-concave.
\end{proposition}

\begin{proof}
We first show that $g_p$ is log-concave. The hessian $\*H = [h_{ij}] = \nabla^2 \log g_p$ is
\begin{align*}
\*H(x, y, z) = -\frac{1}{3 \left(g_p(x, y, z) \right)^2} \*R(x, y, z),
\end{align*}

where
\begin{align*}
\*R(x, y, z) =
\begin{bmatrix} 
3(y + z + 1)^2 & 3z^2 + 3z - 1 & 3y^2 + 3y - 1 \\[0.5em]
3z^2 + 3z - 1 & 3(x + z + 1)^2 & 3x^2 + 3x - 1 \\[0.5em]
3y^2 + 3y - 1 & 3x^2 + 3x - 1 & 3(x + y + 1)^2
\end{bmatrix}.
\end{align*}

We will show that $\*H$ is negative definite for all $x, y, z \in \mathbb{R}_{>0}^3$ by showing that $\*R$ is positive definite for all $x, y, z \in \mathbb{R}_{>0}^3$.
Note that
\begin{align*}
|r_{12}| + |r_{13}| &= |3z^2 + 3z - 1| + |3y^2 + 3y - 1|\\
                    &\leq 3z^2 + 3z + 1 + 3y^2 + 3y + 1\\
                    &< 3z^2 + 3y^2 + 6yz + 6y + 6z + 3 = |r_{11}|.
\end{align*}
Completely analogously, it is easy to see that $|r_{21}| + |r_{23}| < |r_{22}|$ and $|r_{31}| + |r_{32}| < |r_{33}|$.
Therefore, $\*R$ is strictly diagonally dominant.
Since $r_{ii} > 0$, for $i \in \{1, 2, 3\}$, it follows that $\*R$ is positive definite (see Theorem 6.1.10 of \cite{horn12}).

It remains to show that any derivative of $g_p$ is log-concave.
Derivatives of order $\geq 2$ are identically zero, therefore trivially log-concave.
For the first-order derivative $\partial_{x} g_p$, we have
\begin{align*}
\partial_{x} g_p(x, y, z) = \frac{1}{22}(3 + 3y + 3z),
\end{align*}
and
\begin{align*}
\nabla^2(\log \partial_{x} g_p)(x, y, z) = -\frac{1}{(y + z + 1)^2}\*W,
\end{align*}
where
\begin{align*}
\*W =
\begin{bmatrix} 
0 & 0 & 0 \\[0.5em]
0 & 1 & 1 \\[0.5em]
0 & 1 & 1
\end{bmatrix}.
\end{align*}
It is easy to see that $\*W$ has eigenvalues $\lambda_1 = 2$ and $\lambda_2 = \lambda_3 = 0$, therefore $\partial_x g_p$ is log-concave.
Showing log-concavity for $\partial_y g_p$ and $\partial_z g_p$ is completely analogous.
\end{proof}

\section{Illustration}
We consider the parametric family of discrete distributions represented by generating polynomials of the form
\begin{align}
\hat{g}_p(x, y, z) = \frac{1}{4 + 3b + 3c}(4 + b(x+y+z) + c(xy + xz + yz)), \label{eq:polyfam}
\end{align}
for all $x, y, z \in \mathbb{R}$ and $b, c \in \mathbb{R}_{\geq 0}$.

The counterexample presented in the previous section is obtained for $b = c = 3$.
In \figref{fig:slc} we show a simulated approximation of the indicator functions of strong log-concavity and log-submodularity for the above family of distributions.

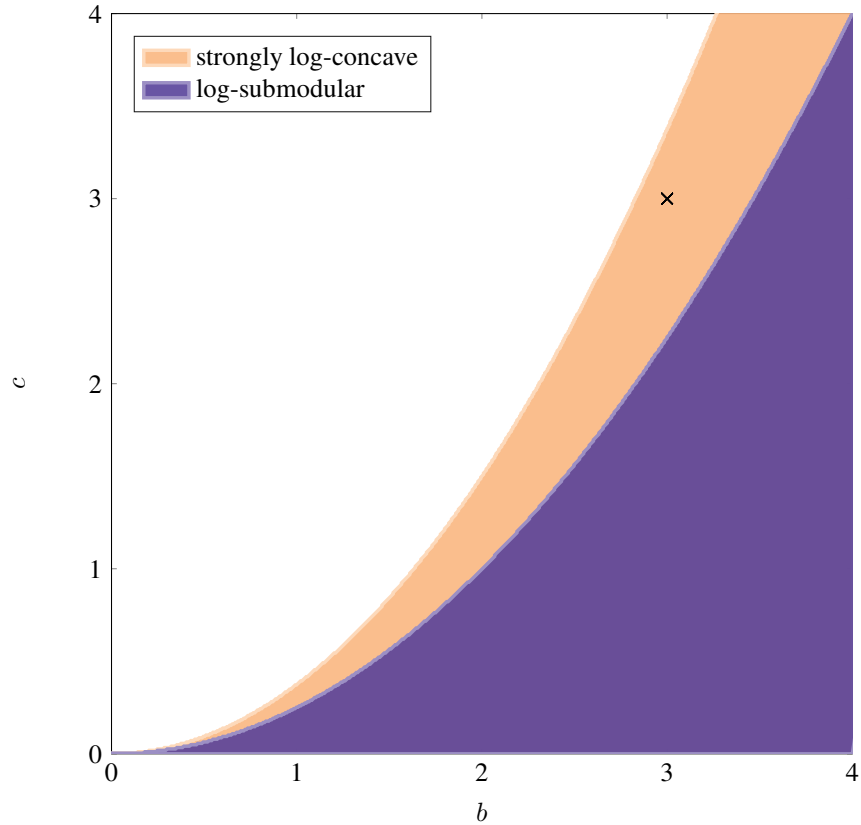
\begin{figure}[tb]
  \centering
\begin{tikzpicture}

\colorlet{col1}{Apricot}
\colorlet{col2}{Violet}

\begin{axis}[%
tick label style={font=\normalsize},
label style={font=\normalsize},
legend style={font=\normalsize},
view={0}{90},
width=4.5in,
height=4.5in,
xmin=0, xmax=4,
xtick={0, 1, 2, 3, 4},
xticklabels={0, 1, 2, 3, 4},
xlabel={$b$},
%xlabel shift=0em,
ymin=0, ymax=4,
ytick={0, 1, 2, 3, 4},
yticklabels={0, 1, 2, 3, 4},
ylabel={$c$},
%ylabel shift=0em,
%tick label style={/pgf/number format/fixed},
major tick length=2pt,
%axis lines*=left,
legend cell align=left,
clip marker paths=true,
legend style={at={(0.03,0.97)},draw=black!90!white,row sep=0,anchor=north west}
%every axis plot/.append style={
%  mark=none,
%  line width=1.7pt,
%  opacity=0.9,
%}
]

\path[name path=axis] (axis cs:0,0) -- (axis cs:4,0);

\addplot [
name path=fslc,
draw=none,
forget plot
] table {slc.txt};

\addplot [
name path=fnlc,
draw=none,
forget plot
] table {nlc.txt};

\addplot [
draw=col1!50!white,
line width=1.3,
fill=col1,
fill opacity=0.9
]
fill between[
of=fslc and axis,
%soft clip={domain=0:1},
];
\addlegendentry{strongly log-concave}

\addplot [
draw=col2!50!white,
line width=1.3,
fill=col2, 
fill opacity=0.9
]
fill between[
of=fnlc and axis,
%soft clip={domain=0:1},
];
\addlegendentry{log-submodular}

\addplot [
draw=black,
mark=x,
mark size=3,
forget plot
] (3,3);

\end{axis}
\end{tikzpicture}
  \caption{A depiction of the parameter ranges for which the resulting distribution in \eqref{eq:polyfam} is strongly log-concave or log-submodular.
  For this family of distributions it is clear that the region of strong log-concavity is a strict superset of the region of log-submodularity.
  The cross indicates the location of the counterexample discussed previously.}
  \label{fig:slc}
\end{figure}

\bibliography{note}
\bibliographystyle{icml2019}

\end{document}